\documentclass[11pt, DIV=9]{scrartcl}
\usepackage[a4paper,margin=1in,tmargin=1in,bmargin=1.25in]{geometry}

\usepackage[T1]{fontenc}
\usepackage[english]{babel}
\usepackage{amsmath, amsfonts, amsthm, amssymb, mathtools, stmaryrd}
\usepackage[pdfusetitle,hidelinks]{hyperref}
\usepackage[capitalise,noabbrev]{cleveref}
\usepackage{csquotes}
\usepackage[osf,sc]{mathpazo}
\usepackage{relsize}
\usepackage{microtype}
\usepackage{pdfpages}

\usepackage{pgfplots}

\usepackage{thm-restate}
\usepackage{subcaption}

\usepackage{xcolor}

\usepackage{graphicx}

\usepackage{complexity}

\usepackage{abstract}
\usepackage[inline]{enumitem}
\usepackage{booktabs}
\usepackage{todonotes}

\usepackage[linesnumbered,ruled,commentsnumbered,longend,]{algorithm2e}
\usepackage{pifont}

\usepackage{tcolorbox}
\SetKwProg{Fn}{function}{}{end}
\SetKwInOut{Input}{\hspace*{1.2em}Input\hspace*{0em}}
\SetKwInOut{Output}{\hspace*{0.4em}Output\hspace*{0em}}
\SetKwInOut{Auxiliary}{\hspace*{-0.5em}Auxiliary\hspace*{0em}}

\makeatletter
\definecolor{codebg}{cmyk}{0,0,0,0}
\def\@algocf@pre@ruled{\begin{tcolorbox}[colback=codebg,arc=0.3em,boxsep=0em,left=0em, right=0.5em, top=0.5em, bottom=0.5em, boxrule=0.1mm]}%
\def\@algocf@post@ruled{\end{tcolorbox}\vspace*{-1.5em}}%
\renewcommand{\fnum@algocf}{\hspace*{0.5em}\AlCapSty{\AlCapFnt\algorithmcfname} \arabic{algocf}}
\makeatother

\usepackage{todonotes} 

\usepackage{tikz}
\usepackage{tikz}
\usepackage{tikz-layers}
\usetikzlibrary{calc}
\usetikzlibrary{decorations.pathmorphing}
\usetikzlibrary{positioning, arrows.meta, fit}
\usetikzlibrary{shapes}
\usetikzlibrary{arrows.meta,backgrounds,automata}
\usetikzlibrary{matrix,calc}
\tikzstyle{vertex}=[circle, draw, fill=gray!80!white,thick,scale=1.2]
\tikzstyle{edge}=[draw=black, thick,-]
\usetikzlibrary{hobby,arrows,decorations.pathmorphing,backgrounds,positioning,fit,petri,calc}
\tikzstyle{vertex}=[anchor=center, circle, fill=gray, inner sep=1.75]

\usetikzlibrary{cd}
\usepackage{microtype}

\theoremstyle{definition}
\newtheorem{definition}{Definition}[section]

\newtheorem{remark}[definition]{Remark}

\theoremstyle{plain}
\newtheorem{lemma}[definition]{Lemma}
\newtheorem{corollary}[definition]{Corollary}
\newtheorem{theorem}[definition]{Theorem}

\newtheorem{problem}[definition]{Problem}

\Crefname{fact}{Fact}{Facts}

\tikzstyle{place}=[circle,draw=black!80,thick,fill=black!80, inner sep=0pt,minimum size=1.5mm]

\newcommand{\mgG}{\gamma}

\renewcommand\phi\varphi
\renewcommand\epsilon\varepsilon

\DeclareMathOperator{\Sym}{Sym}

\DeclareMathOperator{\Aut}{Aut}
\DeclareMathOperator{\Var}{Var}

\DeclareMathOperator{\bip}{bip}

\DeclareMathOperator{\dcup}{\dot{\cup}}

\newcommand{\n}[1]{\overline{#1}}
\newcommand{\Orbi}[2]{{#2}^{#1}}
\newcommand{\xparagraph}[1]{\textbf{#1}}

\usepackage{authblk}

\title{The Complexity of Symmetry Breaking Beyond Lex-Leader}
\author[1]{Markus Anders}
\author[1]{Sofia Brenner}
\author[2]{Gaurav Rattan}

\affil[1]{TU Darmstadt}
\affil[2]{University of Twente}
\affil[ ]{\small\textit{\{anders,brenner\}@mathematik.tu-darmstadt.de, g.rattan@utwente.nl}}

\usetikzlibrary{arrows.meta}

\usepackage{xspace}

\newsavebox{\fminibox}
\newlength{\fminilength}

\newcommand{\dCANON}{$d$-\textsc{scanon}\xspace}
\newcommand{\sCANON}{$s$-\textsc{scanon}\xspace}
\newcommand{\dGCANON}{$d$-\textsc{gcanon}\xspace}

\newcommand{\dCANONf}{$d$-\textsc{scanon}$_F$\xspace}
\newcommand{\sCANONf}{$s$-\textsc{scanon}$_F$\xspace}
\newcommand{\dGCANONf}{$d$-\textsc{gcanon}$_f$\xspace}
\newcommand{\sGCANONf}{$s$-\textsc{gcanon}$_f$\xspace}

\begin{document}
\maketitle

\begin{abstract}
Symmetry breaking is a widely popular approach to enhance solvers in constraint programming, such as those for SAT or MIP. Symmetry breaking predicates (SBPs) typically impose an order on variables and single out the lexicographic leader (lex-leader) in each orbit of assignments. Although it is NP-hard to find complete lex-leader SBPs, incomplete lex-leader SBPs are widely used in practice.

In this paper, we investigate the complexity of computing complete SBPs, lex-leader or otherwise, for SAT. Our main result proves a natural barrier for efficiently computing SBPs: efficient certification of graph non-isomorphism. Our results explain the difficulty of obtaining short SBPs for important CP problems, such as matrix-models with row-column symmetries and graph generation problems. Our results hold even when SBPs are allowed to introduce additional variables. We show polynomial upper bounds for breaking certain symmetry groups, namely automorphism groups of trees and wreath products of groups with efficient SBPs.
\end{abstract}

\section{Introduction} \label{sec:intro}
The search space of a constraint program can exhibit a large amount of symmetry. 
This simple yet far-reaching observation forms the core principle behind the use of \emph{symmetry based} approaches in the realm of constraint programming \cite{DBLP:reference/fai/GentPP06, DBLP:series/faia/Sakallah21}. 
Such methods prune the symmetric parts of the search space to save computational costs. 
Ideally, they ensure that at most one solution exists per equivalence class of candidate solutions. 
Over the last two decades, numerous methods have been proposed to exploit symmetries of constraint programs.
In particular, many approaches have been developed for Boolean satisfiability solvers \cite{DBLP:conf/kr/CrawfordGLR96,DBLP:conf/dac/AloulMS03,DBLP:conf/cp/CodishGIS16,DBLP:conf/sat/Devriendt0BD16,DBLP:journals/jsc/JunttilaKKK20,DBLP:conf/cp/KirchwegerS21,DBLP:conf/sat/Devriendt0B17,DBLP:journals/constraints/Sabharwal09,DBLP:conf/tacas/MetinBCK18} as well as mixed integer programming \cite{DBLP:journals/mp/Margot02, DBLP:conf/ipco/OstrowskiLRS08, DBLP:journals/mpc/PfetschR19}.
Symmetry-based solving remains an active and fruitful area of interest, especially from a practical perspective: for example, the defining feature of arguably one of the most successful entries in the SAT competition 2023 was symmetry breaking \cite{breakidkissatcomp, DBLP:journals/jair/BogaertsGMN23}.

How symmetries should be used best remains unclear.
Approaches can be roughly divided into two different categories: in \emph{dynamic} and \emph{static} approaches.
In a \emph{dynamic} approach, symmetries are used during the execution of a solver \cite{DBLP:journals/constraints/Sabharwal09, DBLP:conf/ictai/DevriendtBCDM12, DBLP:conf/sat/Devriendt0B17, DBLP:conf/tacas/MetinBCK18}. 
A typical example is that the solver incorporates a branching rule that makes use of the symmetries directly \cite{DBLP:journals/constraints/Sabharwal09}.

The second approach is the \emph{static} use of symmetries, which is the main focus of this paper.
Here, additional constraints, so-called \emph{symmetry breaking predicates} (SBPs), are added to a given problem instance.
The notion of SBPs was first introduced in the seminal paper of Crawford, Ginsberg, Luks and Roy~\cite{DBLP:conf/kr/CrawfordGLR96}. 
Their goal was to generate polynomial-sized SBPs for SAT formulas in conjunctive normal form (CNF).
However, since their framework is rooted in group theory, many results neatly generalize to other constraint languages. 

The framework of Crawford et al., as well as the majority of the subsequent work in this area, uses so-called \emph{lex-leader} predicates to achieve complete symmetry breaking. 
Using \emph{incomplete} lex-leader predicates is arguably one of the most successful approaches to symmetry breaking in practice \cite{DBLP:conf/dac/AloulMS03, DBLP:conf/sat/Devriendt0BD16}.
On a complexity-theoretic level, however, Crawford et al.~proved that computing a predicate true of only the lex-leader in each equivalence class of Boolean assignments is \NP-hard. 
Subsequent results showed that this even holds true for restricted classes of groups \cite{DBLP:journals/amai/LuksR04}, as well as orders similar to lex-leader \cite{DBLP:conf/cp/KatsirelosNW10, DBLP:journals/corr/abs-2005-08954}. 

One may wonder whether there are \emph{other kinds} of SBPs that are efficiently computable. 
Here, \emph{other kinds} of SBPs simply means that they do \emph{not} make use of a lexicographic ordering of the assignments. 
In principle, choosing any canonical representative among symmetric assignments is permissible, lex-leader or otherwise.
This question is motivated, for instance, by the realm of graph isomorphism (GI). 
There, choosing the lex-leader is also known to be $\NP$-hard \cite{DBLP:conf/stoc/BabaiL83}, and the best theoretical and practical approaches make use of other mechanisms.

Concerning practical symmetry breaking, only a few, though surprisingly different, approaches of generating non-lex-leader SBPs have been explored. 
In \cite{DBLP:journals/amai/FlenerPS09}, the global cardinality constraint \cite{DBLP:conf/aaai/Regin96} is used in conjunction with lex-leader constraints to efficiently handle (particular) wreath symmetry. 
In \cite{DBLP:conf/cp/CodishGIS16}, SAT symmetry breaking constraints for graph problems are produced similarly to the canonical labeling algorithm \textsc{nauty} \cite{DBLP:journals/jsc/McKayP14}.
In \cite{DBLP:journals/mics/Heule19}, minimal SAT symmetry breaking constraints are generated for small groups.

In general, however, the complexity of SBPs remains largely unexplored, even for fairly restricted kinds of symmetries. 
Perhaps the most glaring example is the problem of breaking \emph{row-column symmetries}, which arise in the so-called \emph{matrix models} \cite{DBLP:conf/cp/FlenerFHKMPW02}. 
These models allow the decision variables to be arranged in a matrix such that interchanging any two rows or any two columns is a symmetry of the model. Matrix models arise in multiple areas of constraint programming such as scheduling, combinatorial problems, and design \cite{matrixmodelling}. 
Perhaps the most well-known matrix model is the pigeonhole principle problem, for which it is \NP-hard to compute the lex-leader or similar assignments \cite{DBLP:conf/kr/CrawfordGLR96}.
While the problem of devising SBPs for such models has received much attention \cite{DBLP:conf/cp/FlenerFHKMPW02,DBLP:conf/cp/KatsirelosNW10, DBLP:reference/fai/GentPP06}, the known results do not explain the lack of compact SBPs for matrix models.

\subsection{Our Results} \label{sec:contribution}
The objective of this paper is to further investigate the exact complexity of computing static symmetry breaking predicates. 
\emph{Given a group of symmetries on the variables of a formula, how hard is it to generate a complete symmetry breaking predicate?} 
The ultimate goal of our work is to obtain a classification of symmetry groups, in terms of the complexity of computing SBPs. Such a classification could help inform practitioners as to which cases can be handled easily, and which ones are more challenging. 

In order to simplify the exposition, our setting of choice is that of Boolean satisfiability testing (SAT).
However, in the same vein as \cite{DBLP:conf/kr/CrawfordGLR96}, our results are founded in a general group-theoretic setting, so they should easily transfer to many branches of constraint programming: 
we consider computing symmetry breaking predicates for a given \emph{permutation group}, instead of a particular SAT formula exhibiting such symmetry.

Our results can be divided into \emph{hardness results} and \emph{upper bounds}. The high-level idea for the hardness results can be summarized as follows: We show that if symmetry breaking is feasible for certain expressive groups, such as matrix groups or Johnson actions, then graph isomorphism is in \coNP{}. The containment \GI{} $\in$ $\mathsf{co}$-$\mathsf{NP}$ is a major unresolved problem \cite{KoeblerST93}, even for the restricted case of group isomorphism \cite{DBLP:conf/stoc/Babai16}.
While \GI{} $\in$ $\mathsf{co}$-$\mathsf{NP}$ seems to have no other major complexity theoretic consequences and is seemingly not ``implausible'', it still poses a barrier to compact SBPs.

The idea of our reductions is to encode the input graphs as binary strings in a suitable way, then guess a canonizing permutation, and use the symmetry-breaking constraint to verify that the result is indeed the canonical form: By definition, the symmetry-breaking constraint is true of precisely the canonical forms. The graphs are non-isomorphic exactly when the canonical forms are different. As a strengthening, we show that this holds even when the symmetry breaking constraint uses additional variables as their values can be guessed as well.

We now explain our hardness results in greater detail. 

\xparagraph{Matrix Models.} Our first result tackles the difficulty of breaking row-column symmetries in matrix models. 
As mentioned above, this problem has received much attention in symmetry breaking literature. 

\begin{restatable}{theorem}{gridhard}\label{thm:core}
	Suppose there exists a polynomial time algorithm for generating complete symmetry breaking predicates for row-column symmetries. Then \GI{} $\in$ $\mathsf{co}$-$\mathsf{NP}$ holds, i.e., graph non-isomorphism admits a non-deterministic polynomial time algorithm. 
\end{restatable}

Our theorem explains the difficulty of obtaining compact symmetry breaking predicates for matrix models, in the sense that it would imply polynomial time algorithms for certifying graph non-isomorphism. 
Section~\ref{sec:gridthm} contains a detailed description of our result. 

\xparagraph{Johnson Actions.} 
We identify yet another class of groups for which symmetry breaking is hard, namely the $(k,t)$-\emph{Johnson groups}. These are symmetric groups $\Sym(k)$ acting on $t$-subsets of $[k]$ for fixed $t<k$. 
It is well-known that these actions form an important sub-case of Babai's quasi-polynomial algorithm for graph isomorphism \cite{DBLP:conf/stoc/Babai16}. 
\begin{restatable}{theorem}{johnson}\label{thm:johnson}
	Let $t>1$ be a fixed positive integer. Suppose that we can generate complete symmetry breaking predicates for all $(k,t)$-Johnson groups in polynomial time (in terms of the domain size). Then $\GI{} \in \coNP{}$  holds.
\end{restatable}
Section~\ref{sec:johnson} contains a formal description of this result. In fact, it follows from Theorem~\ref{theo:johnson3}, which proves a stronger statement.

\xparagraph{Certificates for Canonization.} 
We strengthen our hardness results of Theorem~\ref{thm:core} and Theorem~\ref{thm:johnson} as follows.
We allow an algorithm to produce more expressive SBPs:
\begin{enumerate}
	\item The SBP can be given as a Boolean circuit.
	\item The SBP is allowed to \emph{introduce additional variables}. Essentially, this gives the predicate access to additional non-determinism. 
	The SBP may introduce an arbitrary number of additional variables, as long as the overall size is polynomial. 
\end{enumerate}

Despite allowing more powerful SBPs, we conclude a stronger hardness implication: an efficient algorithm for such predicates implies an \emph{efficiently verifiable graph canonical form} (see Theorem~\ref{lem:gridhard} and Theorem~\ref{theo:johnson2}). 
Note that an efficient verifier for graph canonical forms implies an efficient verifier for graph non-isomorphism (Lemma~\ref{lem:dcanonnp}), but the converse is unknown.
For this result, we observe that SBPs for a permutation group $G$ on the domain $[n]$ essentially solve a particular \emph{decision version} of the \emph{string canonization} problem w.r.t.~$G$ on strings of length $n$. 
String canonization is a fundamental problem of interest in the graph isomorphism community \cite{DBLP:conf/stoc/BabaiL83,DBLP:conf/stoc/Babai16,DBLP:conf/stoc/Babai19}.
Section~\ref{sec:dcanon} contains a detailed description.
Moreover, we prove that the hardness results also hold for all subgroups of polynomial index (see Lemma~\ref{lemma:index}).

\xparagraph{Quasi-Polynomial Upper Bound.}
Realizing that symmetry breaking reduces to string canonization allows us to express an upper-bound on the size of \emph{circuit} SBPs for general permutation groups. 
The result is mainly of theoretical interest, but we believe that this could have useful consequences in SBP heuristic design. 
The theorem immediately follows from the quasi-polynomial time algorithm of Babai \cite{DBLP:conf/stoc/Babai19}, see Section~\ref{sec:dcanon} for more details. 
\begin{restatable}{theorem}{quasiupper}\label{thm:quasiupper}
	Given a permutation group $G \leq \Sym(n)$, there is a quasi-polynomial time (in $n$) algorithm producing a complete symmetry breaking circuit of quasi-polynomial size. 
\end{restatable}

We complement these results concerning \emph{hard} families by focusing on polynomial upper bounds, i.e., the question for which families of groups symmetry breaking is \emph{easy}:

\xparagraph{Polynomial Upper Bounds.} 
In Section~\ref{sec:upperbounds}, we examine how group-theoretic structure can help to design SBPs. 
Our results show how we can assemble SBPs for a group from the SBPs of its constituents, in context of natural operations such as disjoint direct products and wreath products. This extends the results of \cite{Grayland2009}, where the existence of \emph{lex-leader} constraints for constituents is assumed to assemble constraints for direct products and wreath products. (The paper also treats cyclic, dihedral and alternating groups.)

The following theorem is the main consequence of our results in this section.

\begin{restatable}{theorem}{treeintro}\label{thm:tree-intro}
	Assume that $G \leq \Sym(n)$ is the automorphism group of a tree $T$. Then $G$ admits a complete symmetry breaking predicate of linear size. Given the tree $T$, it can be computed in polynomial time.
\end{restatable}

Automorphism groups of trees are special cases of so-called wreath products.
Such groups naturally occur, for example, whenever models exhibit hierarchical structure.
Intuitively, the structure can be split into parts with the same symmetry group (the \emph{base group}), which are permuted by the so-called \emph{top group}. Essentially, we combine symmetry breaking constraints for the base group and the top group to a symmetry breaking constraint for the wreath product by using the predicate for base group to make every part canonical, and the constraint of the top group to fix an ordering of the parts. 
For the general case of wreath products the problem is far more technical, but we obtain the following result (see Section~\ref{sec:upperbounds} for details). 
\begin{restatable}{theorem}{wreath}\label{thm:wreathproducts}
	Let $G \leq \Sym(n)$ and $H \leq \Sym(m)$ be permutation groups. Assume that a complete symmetry breaking circuit for $G$ can be computed in polynomial time. Moreover, suppose that for every partition $P$ of $[m]$, the partition stabilizer $S$ of $P$ in $H$ and a complete symmetry breaking circuit for $S$ can be computed in polynomial time. Then there is a complete symmetry breaking circuit for the wreath product $W\coloneqq G \wr H$ that can be computed in polynomial time.
\end{restatable} 

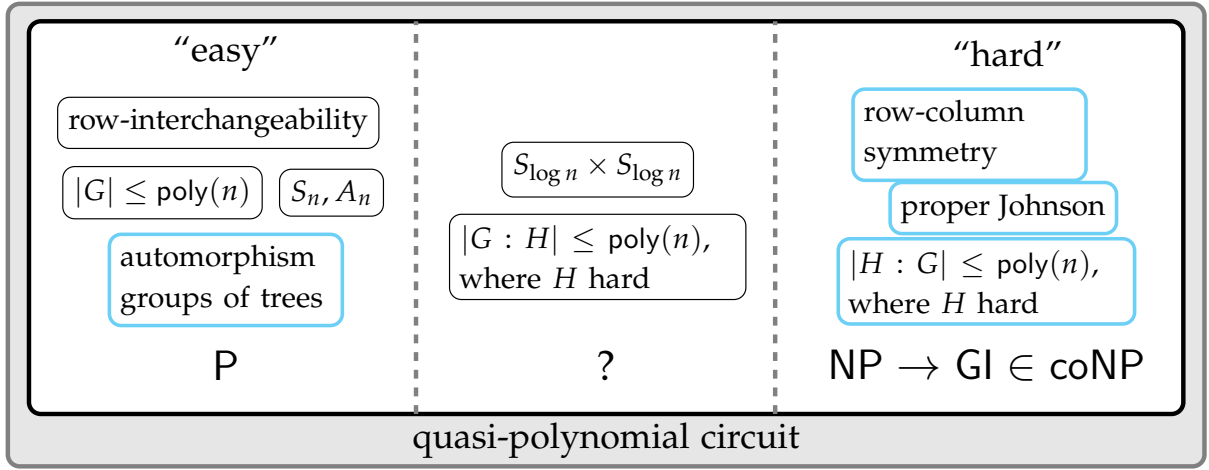
\begin{figure}[t]
	\centering
	\resizebox{\textwidth}{!}{
	\begin{tikzpicture}[yscale=0.833]
		\draw[draw=black!50,fill=black!10,rounded corners,line width=1.5pt] (-1.33-0.25, -1.25) rectangle (12.33+0.25, 5.125+0.25) {};
		\draw[rounded corners,fill=white,line width=1.5pt] (-1.33, -0.5) rectangle (12.33, 5.125) {};
		\node[] at (5.5,-0.87) {\large quasi-polynomial circuit};

		\begin{scope}
			\node[] at (1,4.66) {\large ``easy''};

			\node[draw, rounded corners] at (0+0.25,3-0.33) {\small $|G| \leq \poly(n)$};
			\node[draw=cyan!50, rounded corners,very thick,text width=2.5cm] at (0.75+0.25,1.75-0.33) {\small automorphism groups of trees};
			\node[draw, rounded corners] at (2+0.25,3-0.33) {\small $S_n, A_n$};
			\node[draw, rounded corners] at (0.66+0.25,4-0.33) {\small row-interchangeability};

			\node[] at (1,0.2) {\Large \P{}};
		\end{scope}

		\draw[line width=1.5pt,draw=black!50,dashed] (3.25, -0.5) to (3.25, 5.125);

		\begin{scope}[xshift=4cm]
			\node[draw, rounded corners] at (1.4,3) {\small $S_{\log n} \times S_{\log n}$};
			\node[draw, rounded corners,text width=3.25cm] at (1.4,1.75) {\small $|G : H| \leq \poly(n)$, where $H$ hard};
			\node[] at (1.5,0.2) {\Large ?};
		\end{scope}

		\draw[line width=1.5pt,draw=black!50,dashed] (7.5, -0.5) to (7.5, 5.125);

		\begin{scope}[xshift=8cm]
			\node[] at (2.25,4.66) {\large ``hard''};

			\node[draw=cyan!50, rounded corners,very thick,text width=2.5cm] at (1.8,3.5) {\small row-column symmetry};
			\node[draw=cyan!50, rounded corners,very thick] at (2.2,2.45) {\small proper Johnson};
			\node[draw=cyan!50, rounded corners,very thick,text width=3.25cm] at (2,1.4) {\small $|H : G| \leq \poly(n)$, where $H$ hard};
			\node[] at (2,0.2) {\Large \NP{} $\to$ \GI{} $\in$ \coNP{}};
		\end{scope}
	\end{tikzpicture}}
	\caption{Complexity of computing symmetry breaking predicates for the stated families of groups in SAT. All groups can be handled in quasi-polynomial time using a circuit. The symbol~$G$ refers to the permutation group of consideration. The parameter $n$ refers to the domain size of the permutation group, or, the number of variables of the formula. For ``easy'' families of groups, a CNF predicate can be computed in polynomial time. For ``hard'' families of groups, the existence of polynomial time symmetry breaking, even allowing the use of additional variables, implies that \GI{} is in \coNP{}. Blue outlines indicate novel results proven in this paper.}
	\label{fig:overview}
\end{figure}
In summary, Figure~\ref{fig:overview} provides a concise description of our progress towards a complexity classification for the problem of generating SBPs for permutation groups.

\section{Preliminaries} \label{sec:preliminaries}

\subsection{Boolean Circuits and Satisfiability} \label{sec:sat}
\xparagraph{Boolean Circuits.} A Boolean circuit $\psi$ is a circuit consisting of input gates, one output gate, and $\{\text{AND}, \text{OR}, \text{NOT}\}$-gates connecting them in the usual way. 
We refer to the input gates as the \emph{variables} $\Var(\psi)$.
The \emph{size} of a circuit refers to the number of gates.   

An \emph{assignment} of $\psi$ is a function $\theta \colon V \to \{0, 1\}$ where $V \subseteq \Var(\psi)$.
The assignment is \emph{complete} whenever $V = \Var(\psi)$ and \emph{partial} otherwise.
A circuit is evaluated using an assignment $\theta \colon V \to \{0, 1\}$, by replacing each input gate $v \in V$ with $\theta(v)$, with the usual meaning. 
The resulting circuit is $\psi[\theta]$.
Whenever $\theta$ is complete, the value of the output gate can be determined in linear time, and hence either $\psi[\theta] = 0$ or $\psi[\theta] = 1$ holds.

If $\psi[\theta] = 1$ we call $\theta$ a \emph{satisfying assignment}, whereas if $\psi[\theta] = 0$ we call $\theta$ a \emph{conflicting assignment}. 
A circuit $\psi$ is \emph{satisfiable} if and only if there exists a satisfying assignment to $\psi$.

\xparagraph{Conjunctive Normal Form.} In practice, a SAT instance $\psi$ is typically given in \emph{conjunctive normal form} (CNF), which we denote with
$\psi = \{\{l_{1,1} \vee \cdots{} \vee l_{1,k_1}\} \wedge \cdots{} \wedge \{l_{m,1} \vee \cdots{} \vee l_{m,k_m}\}\}.$
Each element $C \in \psi$ is called a \emph{clause}, whereas a clause itself consists of a set of \emph{literals}. 
A literal is either a variable $v$ or its negation $\n{v}$.

A symmetry, or \emph{automorphism}, of $\psi$ is a permutation of the variables $\varphi \colon \Var(\psi) \to \Var(\psi)$ which maps $\psi$ back to itself, i.e., $\psi^\varphi \equiv \psi$, where $\varphi$ is applied element-wise to the variables in each clause. 
The permutation group of all symmetries of $\psi$ is $\Aut(\psi) \leq \Sym(\Var(\psi))$.

Another common way to define symmetries is to define them on the \emph{literals} of the formula, allowing the use of so-called \emph{negation symmetries} (see \cite{DBLP:series/faia/Sakallah21}).
In any case, symmetries can be efficiently computed in practice using state-of-the-art symmetry detection tools \cite{DBLP:journals/jsc/McKayP14, DBLP:conf/tapas/JunttilaK11, DBLP:conf/dac/DargaLSM04, DBLP:conf/esa/AndersS21}.

\subsection{Permutation Groups} \label{sec:permutationgroups}
We briefly introduce some notation and results for permutation groups. 
For further background material on permutation groups, we refer to \cite{DIX96}. Throughout, we use the notation $[n] \coloneqq \{1, \ldots, n\}$ for $n \in \mathbb{Z}_{>0}$ and set $[0] \coloneqq \emptyset$.

Let $\Omega$ be a nonempty finite set. Let $\Sym(\Omega)$ denote the \emph{symmetric group} on $\Omega$, i.e., the group of permutations of~$\Omega$.
A \emph{permutation group} is a subgroup~$G$ of $\Sym(\Omega)$, denoted by $G \leq \Sym(\Omega)$. 
We also say that~$G$ \emph{acts on} $\Omega$. 
A permutation group is always specified by the abstract isomorphism type of $G$ (for instance, $G$ could be cyclic of order 10), together with the action of~$G$ on~$\Omega$. For $g \in G$ and $\omega \in \Omega$, we write $\omega^g$ for the image of $\omega$ under $g$ and $\Orbi{G}{\omega} = \{\omega^g \colon g \in G\}$ for the \emph{orbit} of $\omega$ under $G$. 
The \emph{support} of $G$ consists of those elements in~$\Omega$ that are moved (i.e., not fixed) by some element of $G$. For a partition $P = (\Omega_1, \dots, \Omega_r)$ of $\Omega$ (i.e., $\Omega = \Omega_1 \dcup \cdots \dcup \Omega_r$), the \emph{partition stabilizer} of $P$ in $G$ consists of all elements $g \in G$ that setwise stabilize $\Omega_1, \dots, \Omega_r$, i.e.~for all $i \in [r]$, $\{\omega^g \,\!:\,\! \omega \in \Omega_i\}= \Omega_i$.  The \emph{index} of a subgroup $H$ of $G$ is $|G:H| \coloneqq |G|/|H|$. 

Two permutation groups $G \leq \Sym(\Omega)$ and $H \leq \Sym(\Delta)$ are \emph{permutation isomorphic} if there exists a bijection $\lambda \colon \Omega \to \Delta$ and a group isomorphism $\alpha \colon G \to H$ such that $\lambda(\omega^g) = \lambda(\omega)^{\alpha(g)}$ for all $\omega \in \Omega$ and $g \in G$. Note that this notion is stronger than $G$ and $H$ being isomorphic (as abstract groups) as the same abstract group can give rise to different group actions. For instance, $\Sym(k)$ admits so-called Johnson actions on different domains:

\xparagraph{Johnson Groups.} Let $k$ be a positive integer and $t \in [k-1]$. 
A permutation $\pi \in \Sym(k)$ induces a permutation $\hat{\pi}$ on the domain ${[k] \choose t}$ of $t$-subsets of $[k]$, mapping a $t$-subset $A$ to $A^{\hat{\pi}} = \{a^\pi \colon a \in A \}$. 
This way, $\Sym(k)$ becomes a permutation group $S_k^{(t)}$ on a domain of size~${k \choose t}$. The groups $S_k^{(t)}$ are called \emph{Johnson groups} and 
the action is called a \emph{Johnson action}. 
We call a Johnson group \emph{proper} if $t \not\in \{1,k-1\}$ holds. 

Usually, the analogous action of the so-called alternating groups is also called a Johnson action. 
Due to our results in Section~\ref{sec:babai}, it suffices to only consider the symmetric groups. 

\xparagraph{Wreath products.} 
Let $G \leq \Sym(\Omega)$ and $H \leq \Sym(\Delta)$ be permutation groups. The \emph{wreath product} $G \wr H$ consists of the set $G^{\Delta} \times H$, endowed with the multiplication rule $$\bigl((g^{}_\delta)_{\delta \in \Delta}, h\bigr)\bigl((g'_{\delta})_{\delta \in \Delta}, h'\bigr) = \bigl((g^{}_\delta g'_{\delta^{h^{-1}}})_{\delta \in \Delta}, hh'\bigr).$$ We call $G$ the \emph{base group} and $H$ the \emph{top group}.
The group $G \wr H$ acts on $\Omega \times \Delta$ by $$(\omega, \delta)^{((g_\delta)_{\delta \in \Delta},h)} =  (\omega^{g_{\delta^h}}, \delta^h).$$ 
This action is called the \emph{imprimitive action} of the wreath product. 

\subsection{Graph Isomorphism and String Canonization} \label{sec:gi}
\xparagraph{Graphs.} A finite, undirected graph $\Gamma = (V, E)$ consists of a set of vertices $V \subseteq \mathbb{N}$ and an edge relation $E \subseteq {V \choose 2}$.  
Unless stated otherwise, the set of vertices $V$ is $\{1, \dots{}, n\}$ and $m \coloneqq |E|$ denotes the number of edges.
We may refer to the set of vertices of $\Gamma$ with $V(\Gamma)$, and to the set of edges with $E(\Gamma)$. 
The \emph{adjacency matrix} of $\Gamma$ is the $n \times n$-matrix $A = (a_{ij})$ with $a_{ij} = 1$ if $\{i,j\} \in E(\Gamma)$, and $a_{ij} = 0$ otherwise. 
Unless stated otherwise, we assume our graphs are given as adjacency matrices.

A graph~$\Gamma$ is \emph{bipartite} if $V(\Gamma) = A \dcup B$ can be partitioned into two independent sets $A = \{a_1, \dots, a_{k}\}$ and $B = \{b_1, \dots, b_\ell\}$. 
In this case, we may obtain an \emph{bipartite adjacency matrix} $M = (m_{ij})$ by setting $m_{ij} = 1$ if $a_{i}$ and $b_j$ are adjacent, and $m_{ij} = 0$ otherwise.

\xparagraph{Lexicographic ordering.} For $\{0,1\}$-strings $\theta, \theta'$ of the same length, we write $\theta \preceq_{\text{lex}} \theta'$ if~$\theta$ is smaller or equal to $\theta'$ with respect to the lexicographic ordering.
Likewise, we define a lexicographic ordering of matrices with entries in $\{0,1\}$ of a fixed size by interpreting them as strings, reading them row by row. 

\xparagraph{Relational Structures.} 
As a generalization of graphs, we define a \emph{$t$-ary relational structure} $R = (U,A)$, where $U$ is a universe and $A$ is a $t$-ary relation on $U$.
A $t$-ary relational structure is \emph{symmetric} if for every $t$-tuple $(u_1,\dots,u_t)\in A$ and for every $\sigma \in \Sym(t)$, it holds that $(u_{\sigma(1)},\dots,u_{\sigma(t)}) \in A$.

\xparagraph{Graph Isomorphism.} Two graphs $\Gamma_1 = (V_1, E_1), \Gamma_2 = (V_2, E_2)$ are said to be \emph{isomorphic}, whenever there exists a bijection $\varphi \colon V_1 \to V_2$ such that 
$\varphi(\Gamma_1) = (V_1^{\varphi}, E_1^{\varphi}) = (V_2, E_2) = \Gamma_2$
holds.
Here, $V_1^{\varphi}$ and $E_1^{\varphi}$ means applying $\varphi$ element-wise to each element in $V_1$, and each element of each tuple in $E_1$, respectively.
We call $\varphi$ an \emph{isomorphism} between $\Gamma_1$ and $\Gamma_2$.
We may write $\Gamma_1 \cong \Gamma_2$ to denote isomorphism.
A corresponding computational problem follows:
\begin{problem}[\GI{}] Given two graphs $\Gamma_1, \Gamma_2$, does $\Gamma_1 \cong \Gamma_2$ hold?
\end{problem}
Regarding certification, it is easy to see that \GI{} is in \NP{}.
On the other hand, graph isomorphism is known to be in \coAM{}, i.e., there are efficient randomized proofs for non-isomorphism \cite{BabaiM88}.
As mentioned in the introduction, whether graph isomorphism is in \coNP{} is a long-standing open problem \cite{KoeblerST93}. 

Analogously, we may define isomorphism for $t$-ary relational structures $R_1 = (U_1, A_1)$ and $R_2 = (U_2, A_2)$: $R_1$ and $R_2$ are \emph{isomorphic} if there exists a bijection $\pi \colon U_1 \to U_2$ such that for every $(u_1,\dots,u_t) \in A_1$, it holds that $(u_1^{\pi}, \dots, u_t^{\pi}) \in A_2$ and vice-versa. 

\xparagraph{String Canonization.} We next define the \emph{string canonization} problem \cite{BabaiL83, Luks82}. The string canonization problem asks, given a permutation group $G \leq \Sym(\Omega)$ and a string $\sigma \colon \Omega \to \Sigma$ on a finite alphabet $\Sigma$, for a canonical representative of $\sigma^{G}$.
In particular, it computes a function $F \colon \mathcal{G} \times \Sigma^\Omega \to \Sigma^\Omega$ where $\mathcal{G}$ denotes the set of all permutation groups $G \leq \Sym(\Omega)$, and for all $\sigma_1, \sigma_2 \in \Sigma^\Omega$ it holds that 
(1) $F(G, \sigma_1) \cong_G \sigma_1$ and
(2) if $\sigma_1 \cong_G \sigma_2$ then $F(G, \sigma_1) = F(G, \sigma_2)$.  
Here, $\cong_G$ means that $\sigma_1$ can be permuted to $\sigma_2$ using an element of $G$.
A corresponding computational problem follows:
\begin{problem}[\sCANONf] Given a permutation group $G \leq \Sym(\Omega)$, a finite alphabet $\Sigma$ and a string $\sigma \in \Sigma^\Omega$, compute the canonical representative $F(G, \sigma)$.
\end{problem}
The graph isomorphism problem polynomial time reduces to \sCANON, but the converse is unknown. 
However, there is an $F$ such that there is a quasi-polynomial time algorithm which solves the string canonization problem \cite{DBLP:conf/stoc/Babai19}.
It turns out that the string canonization problem is intimately related to symmetry breaking, which we discuss thoroughly in Section~\ref{sec:dcanon}.

A crucial special case of string canonization is \emph{graph canonization}.
As the name suggests, it computes canonical forms for graphs.
Let $f$ be a graph canonization function, i.e., for graphs $\Gamma,\Delta$, it holds that
	(1) $\Gamma \cong \Delta$ iff $f(\Gamma) = f(\Delta)$, and, 
	(2) $f(\Gamma) \cong \Gamma$.
Here, the symbol $\cong$ denotes the graph isomorphism relation. 
The corresponding computational problem follows:
\begin{problem}[\sGCANONf] Given a graph $\Gamma$, compute the canonical representative $f(\Gamma)$ within the isomorphism class of $\Gamma$.
\end{problem}
Indeed, this problem is a special case of string canonization: $G$ can be chosen as a Johnson group of appropriate order and the strings encode the given graphs (see \cite{DBLP:conf/stoc/Babai19}).

\subsection{Notions of Symmetry Breaking}
Next, we define our notions of symmetry breaking. 
Let $\psi$ be a CNF formula. 
Typically, symmetry breaking is defined specifically for the automorphism group $\Aut(\psi)$ of $\psi$.
However, it turns out that often, our symmetry breaking predicates only depend on the structure of $\Aut(\psi)$ and its action on the set of variables $\Var(\psi)$.
In particular, they do not depend on the specific shape of the formula $\psi$.
Hence, we define symmetry breaking only using an arbitrary permutation group $G \leq \Sym(\Omega)$ and without referring to a precise formula $\psi$. 

\xparagraph{Symmetry Breaking Constraints.}
We begin with a discussion of complete symmetry breaking.
Indeed, we find that in the literature two different notions
are in use.

The first of these notions is what we will refer to simply as \emph{complete symmetry breaking}. 
The idea is that a complete symmetry breaking constraint must ensure that in each orbit of \emph{complete} assignments, all but one canonical representative is conflicting \cite{DBLP:conf/kr/CrawfordGLR96}.

Formally, we let $\theta_{\text{full}}(\Omega) \coloneqq \{\theta \;|\; \theta \colon \Omega \to \{0,1\}\}$ denote the set of all complete assignments to $\Omega$.
We let $G \leq \Sym(\Omega)$ act on $\theta_{\text{full}}(\Omega)$ in the natural way.
A Boolean circuit $\psi$ with $\Var(\psi) \subseteq \Omega$ is called a \emph{complete} symmetry breaking circuit for~$G$, whenever for each orbit $O \subseteq \theta_{\text{full}}(\Omega)$ under $G$, there is
\begin{itemize}
	\item a $\tau \in O$ such that $\psi[\tau]$ is satisfying, 
	\item for all $\tau' \in O$ with $\tau \neq \tau'$ the formula $\psi[\tau']$ is conflicting.
\end{itemize}

If $\psi$ is restricted to be a CNF formula, we refer to $\psi$ as a symmetry breaking \emph{predicate}. 
This notion is typically used in the context of general-purpose symmetry breaking, such as for example in \cite{DBLP:conf/kr/CrawfordGLR96, DBLP:conf/dac/AloulMS03, DBLP:conf/sat/Devriendt0BD16, DBLP:journals/mics/Heule19}.
We remark that in \cite{DBLP:journals/mics/Heule19}, this notion is referred to as an \emph{isolator}.

The second notion in use in the literature is \emph{isomorph-free generation}.
It is usually considered in the realm of \emph{dynamic} symmetry breaking.
However, a notion for predicates can be defined: a predicate is supposed to ensure that in each orbit of \emph{partial} assignments, all but one canonical representative is conflicting.
Intuitively, isomorph-free generation ensures that no isomorphic branches are \emph{ever} considered in the search. 
Isomorph-free generation immediately also ensures complete symmetry breaking.
It is typically used in the context of generation tasks, such as in \cite{DBLP:journals/jal/McKay98, DBLP:conf/cp/KirchwegerS21}, but it has also been considered for general-purpose symmetry breaking \cite{DBLP:journals/jsc/JunttilaKKK20}.

The focus of this paper is on complete symmetry breaking and not on isomorph-free generation.

\xparagraph{Symmetry Breaking as a Computational Problem.} 
We define a corresponding computational problem for symmetry breaking.

\begin{problem}[Symmetry Breaking] Given a permutation group $G \leq \Sym(\Omega)$, compute a complete symmetry breaking circuit for $G$.
\end{problem}

There are two variations of this problem that we discuss throughout the paper: 
the first of which concerns the group $G$. 
Usually, $G$ is the automorphism group of a given CNF formula $\psi$, i.e., $G = \Aut(\psi)$.
In this case, the problem might become easier, since automorphism groups and a given formula may admit further structural arguments.
However, considering symmetry breaking for arbitrary permutation groups $G$ opens up the possibility of using symmetries \emph{beyond} syntactic ones, even though it might be unclear how they could be obtained.
Furthermore, results are independent of the specific structure of SAT instances.

The second variation concerns the output: we may expect a CNF predicate, or a Boolean circuit.
Computing a CNF predicate may be harder, since circuits are more expressive.
We believe that all variations of the problem are of potential interest.
Therefore, it seems best to attempt to use the problem definition which yields the strongest possible statement.

\section{Row-Column Symmetries} \label{sec:gridhardness}
In this section, we analyze the complexity of computing symmetry breaking predicates for row-column symmetry.
\cref{sec:gridthm} describes the hardness of obtaining SBPs for breaking row-column symmetries. 
In particular, we provide a proof of \cref{thm:core}. 
\cref{sec:dcanon} establishes the connection between symmetry breaking and decision string canonization. 
Lastly, in \cref{sec:gridthmaux}, we strengthen our results to work for circuit SBPs and SBPs with extra variables.

\subsection{Hardness of Breaking Row-Column Symmetries}\label{sec:gridthm}
We begin with a formal definition of row-column symmetry.

\xparagraph{Row-Column Symmetry.}
Let $m,n$ be two positive integers, and $\Omega \coloneqq [n] \times [m]$.
The \emph{row-column symmetry} group $G$ is defined to be the group $\Sym([n]) \times \Sym([m])$, 
where $\Sym([n])$ naturally acts on the first component of $\Omega$, and $\Sym([m])$ on the second component.
Informally, we can view $\Omega$ as a \emph{matrix} with $n$ rows and $m$ columns. 
The group $G \leq \Sym(\Omega)$ then consists of all the possible row transpositions and all possible column transpositions,
along with their arbitrary compositions.

A \emph{matrix model} is a constraint program whose decision variables can be arranged as a matrix above such that its automorphism group is the row-column symmetry group for this matrix arrangement. For such programs, it is typical to index their variable set by $\Omega = \{x_{ij} \;|\; i \in [n], j \in [m]\}$. 

The following lemma states a one-to-one correspondence between assignments to a matrix model and bipartite graphs. Let $\Gamma(U,V)$ denote a bipartite graph with a designated left-partition $U$ and a right-partition $V$, where $U$ and $V$ are non-interchangeable. 
\begin{lemma}\label{prop:asg_grph}
	There exists a one-to-one correspondence between the set of all Boolean assignments to the variables $\{x_{11}, \dots, x_{nm}\}$ of a matrix model and the set of all bipartite graphs $\Gamma([n],[m])$ with designated left and right partitions. 
\end{lemma}
\begin{proof}
	 Interpret the truth-value of $x_{ij}$ as the indicator for whether there exists an edge between $i \in [n]$ and $j \in [m]$.
\end{proof} 

We proceed with the proof of Theorem~\ref{thm:core}. 

\begin{proof}[Proof of Theorem~\ref{thm:core}]
	We devise a polynomial time verifier for checking purported certificates for non-isomorphism, assuming that we can compute a row-column symmetry breaking predicate in polynomial time.	
	
	\textit{(Bipartite Graphs Suffice.)} It will be more convenient for us to work with bipartite graphs instead of general graphs, in the spirit of standard reductions in isomorphism literature \cite{ZemKNTR}. To every graph $\Gamma$, we can always associate a bipartite graph $\bip(\Gamma)$, namely the vertex-edge incidence graph as follows. 
	The graph $\bip(\Gamma)$ has a designated \emph{left} partition consisting of $V(\Gamma)$, a designated \emph{right} partition consisting of $E(\Gamma)$, and the edges of $\bip(\Gamma)$ are defined by vertex-edge incidence. 
	Moreover, $\bip(\Gamma)$ is a vertex-ordered graph: the left partition inherits the ordering from the graph $\Gamma$, and the right partition $E(\Gamma)$ is ordered according to the ordering induced by the ordering of $V(\Gamma)$.
	Observe that the mapping $\Gamma \mapsto \bip(\Gamma)$ is injective. 
	Moreover, it is easy to verify that two graphs $\Gamma$ and $\Delta$ are isomorphic if and only if the bipartite graphs $\bip(\Gamma)$ and $\bip(\Delta)$ are isomorphic via a bijection which maps the left-partition (right-partition) of $\bip(\Gamma)$ to the left-partition (right-partition) of $\bip(\Delta)$. 
	
	Therefore, it suffices to verify non-isomorphism certificates for bipartite graphs. 
	
	\textit{(Certificate.)} Given two bipartite graphs $\Gamma$ and $\Delta$, our chosen certificate of non-isomorphism is a pair of bijections $(\sigma,\pi)$, where $\sigma \colon V(\Gamma) \to V(\Gamma)$ and $\pi \colon V(\Delta) \to V(\Delta)$. 
	
	\textit{(Verifier.)} Given such a certificate $(\sigma,\pi)$, our polynomial time verifier proceeds as follows: 
	\begin{enumerate}
		\item Compute a symmetry breaking predicate $\delta_{n,m}(x_{11},\dots,x_{nm})$ in time $\mathsf{poly}(n,m)$.
		\item Check if both $\Gamma^\sigma$ and $\Delta^\pi$ satisfy $\delta_{n,m}(x_{11}, \cdots,x_{nm})$, when viewed as Boolean assignments. If both of them satisfy $\delta_{n,m}$, continue; otherwise reject.
		\item Check whether $\Gamma^\sigma \neq \Delta^\pi$, otherwise reject. 
		\item Declare $\Gamma$ and $\Delta$ to be non-isomorphic.
	\end{enumerate}
	It is easy to verify that all of the steps above are polynomial time computations. 
	
	\textit{(Correctness of Verifier.)} 
	It remains to be shown that (1) for every pair of non-isomorphic graphs, there exists a polynomial sized certificate accepted by the verifier above, and (2) for every pair of isomorphic graphs, the verifier always rejects any certificate. 
	
	For (1), let $\Gamma$ and $\Delta$ be two non-isomorphic graphs. 
	Let $\Gamma^*$ be the unique satisfying assignment of $\delta_{n,m}$ in the orbit of $\Gamma$ (similarly define $\Delta^*$) under row-column symmetries.
	Let $\sigma$ be an isomorphism from $\Gamma$ to $\Gamma^*$ (similarly define $\pi$). 
	Since $\Gamma \not\cong \Delta$, it must hold that $\Gamma^*$ and $\Delta^*$ lie in different orbits, and hence $\Gamma^* \neq \Delta^*$. 
	Since the certificate satisfies all conditions of the verifier, the verifier correctly certifies $\Gamma$ and $\Delta$ to be non-isomorphic. 
	
	For (2), suppose $\Gamma$ and $\Delta$ are isomorphic. Then they lie in the same orbit under the action of row-column symmetry on the Boolean assignments to the matrix model.
	Given any certificate $(\sigma, \pi)$, the requirement of $\Gamma^\sigma$ and $\Delta^\pi$ having to satisfy $\delta_{n,m}$ implies that $\Gamma^\sigma = \Delta^\pi$. 
	But then such a certificate is rejected by the verifier in the third step. 
	Hence, the verifier correctly refuses to certify that $\Gamma$ and $\Delta$ are non-isomorphic. 
\end{proof} 

It is not clear whether the converse of Theorem~\ref{thm:core} holds. 
In fact, even a \P{}-time algorithm for graph isomorphism may not be sufficient to yield symmetry breaking algorithms for row-column symmetries. 
In what follows, we address this situation with a closer examination of the complexity of symmetry breaking. 

\subsection{A Decision Version of String Canonization} \label{sec:dcanon} 
We now introduce a \emph{decision variant} of the string canonization problem, which only decides whether a given string \emph{is} the canonical string:
\begin{problem}[\dCANONf] Given a group $G \leq \Sym(\Omega)$, a finite alphabet $\Sigma$ and a string $\sigma \in \Sigma^\Omega$, decide whether $\sigma = F(G, \sigma)$ holds, i.e., whether $\sigma$ is the canonical representative within its isomorphism class $\sigma^G$.
\end{problem}

Let us consider a CNF formula $\psi$.
We consider the case of the string canonization problem where $\Sigma = \{0, 1\}$ and the group $G \leq \Sym(\Var(\psi))$ consists of symmetries of $\psi$. 
Note that any two given assignments $\sigma_1$ and $\sigma_2$ of $\psi$ can be interpreted as strings, and $\sigma_1 \cong_{G} \sigma_2$ holds if and only if they are in the same orbit of $G$.

We observe that an algorithm for \dCANON accepts precisely one assignment per orbit of~$G$.
But this just means that if we translate such an algorithm into a Boolean circuit, the resulting circuit \emph{is} a symmetry breaking circuit. 

Clearly, \dCANONf polynomial time reduces to \sCANONf. 
Since \sCANON can be solved using a quasi-polynomial time algorithm \cite{DBLP:conf/stoc/Babai19}, Theorem~\ref{thm:quasiupper} follows immediately.

Analogously, we may define a decision version of the graph canonization problem, denoted as \dGCANONf.
Recall that graph canonization is a special case of string canonization.
In the following, we prove that the decision canonization problem is tightly related to graph isomorphism in terms of its non-deterministic complexity. 
\begin{lemma}\label{lem:dcanonnp} 
	Let $f$ be a canonical form such that \dGCANONf is in \NP{}. Then, $\GI{} \in \coNP{}$.
\end{lemma}
\begin{proof}
	Assuming \dGCANON is in \NP{} gives us access to a class of polynomial-sized certificates and a polynomial time verifier for these certificates, such that the following hold. 
	If a given graph $\Gamma$ is the canonical representative of its isomorphism class, then there must be a certificate $\sigma$ such that the verifier accepts $(\Gamma, \sigma)$. 
	If $\Gamma$ is not the canonical representative, then for all certificates $\sigma$ the verifier rejects $(\Gamma, \sigma)$.

	Based on this, we provide a non-deterministic polynomial time algorithm for graph non-isomorphism of two graphs $\Gamma$ and $\Delta$.
	
	\textit{(Certificate.)} The certificate consists of two permutations $\varphi_1 \in \Sym(V(\Gamma)), \varphi_2 \in \Sym(V(\Delta))$, a certificate $\sigma_1$ for decision canonization of $\Gamma$, as well as $\sigma_2$ for decision canonization of $\Delta$.   
	
	\textit{(Verifier.)}
	Given two graphs $\Gamma, \Delta$ and certificate $(\varphi_1, \varphi_2, \sigma_1, \sigma_2)$, the verifier proceeds as follows. (Step~1) Run the decision canonization verifier for $({\Gamma}^{\varphi_1}, \sigma_1)$ and $({\Delta}^{\varphi_2}, \sigma_2)$. If both are accepted, proceed, otherwise reject. (Step~2) Accept if ${\Gamma}^{\varphi_1} \neq {\Delta}^{\varphi_2}$, otherwise reject.

	\textit{(Correctness of Verifier.)} Note that whenever we reach Step~2 of the verifier, the procedure guarantees that ${\Gamma}^{\varphi_1}$ is a canonical form of $\Gamma$ and ${\Delta}^{\varphi_2}$ of $\Delta$.
	Hence, $\Gamma \cong \Delta$ holds if and only if ${\Gamma}^{\varphi_1} = {\Delta}^{\varphi_2}$ holds. 
	It immediately follows that the algorithm accepts if and only if $\Gamma$ and~$\Delta$ are non-isomorphic.  
\end{proof}

\subsection{Hardness of Symmetry Breaking with Additional Variables} 
\label{sec:gridthmaux} 
Consider the situation where one is allowed to use additional variables from a set $\Omega'$ to write down symmetry breaking constraints. In principle, this expands our domain $\Omega$ of variables used to $\Omega \dcup \Omega'$. Since the introduction of new variables $\Omega'$ changes the set of assignments, we need to adjust our definition of complete symmetry breaking.

\xparagraph{Symmetry Breaking with Additional Variables.}
A Boolean circuit $\psi$ is called a complete symmetry breaking circuit \emph{with additional variables} for $G \leq \Sym(\Omega)$, whenever for each orbit $\tau \subseteq \sigma_{\text{full}}(\Omega)$ under ${G}$, there is 
\begin{itemize}
	\item a $\tau' \in \tau$ such that $\psi[\tau']$ is satisfiable, 
	\item for all $\tau'' \in \tau$ with $\tau' \neq \tau''$ the circuit $\psi[\tau'']$ is unsatisfiable.
\end{itemize}
A point of contention in the above definition might be whether $\psi[\tau']$ should actually have \emph{exactly one} satisfying assignment.
This would ensure that there is precisely one satisfying assignment per orbit, while our definition only suffices to ensure a unique satisfying assignment \emph{when restricted} to the variables of $\psi$. In this paper, we stick to the above definition.

Using additional variables is typically not considered in the literature, most likely because this might substantially alter the difficulty of the underlying instance. 
Introducing additional variables is however intriguing: 
it gives the symmetry breaking predicates access to non-determinism, and hence might enable substantially more powerful constraints.

\xparagraph{Hardness with Additional Variables.}
We show hardness results for symmetry breaking \emph{even if} we are allowed to introduce new variables.
\begin{theorem}\label{lem:gridhard}
	Suppose there exists a polynomial time algorithm for generating complete symmetry breaking circuits \emph{with additional variables} for row-column symmetries. Then, it holds that \dGCANON$\in \NP{}$ and hence $\GI{} \in \coNP{}$.
\end{theorem}
\begin{proof}
	It suffices to show that there exists a canonical form $f$ such that \dGCANONf$\in \NP{}$ (see Lemma~\ref{lem:dcanonnp}).

	\textit{(Certificate.)} Given a graph $\Gamma = (V, E)$, we consider the row-column symmetry group with $n = |V|$ rows and $m = |E|$ columns.
	More precisely, we assume $V = [n]$. 
	We let $\delta_{n,m}(x_{11},\dots,x_{nm}, y_1,\dots,y_p)$ denote the symmetry breaking circuit as computed by the polynomial time algorithm for row-column symmetry. Obviously, $p \in \textrm{poly}(m,n)$. 
	Our chosen certificate for decision canonization is $\sigma$, where $\sigma$ is an assignment to the variables $V = \Var(\delta_{n,m}(x_{11},\dots,x_{nm}, y_1,\dots,y_p))$.

	\textit{(Bipartite Reordering.)} We observe a technicality of our reduction: 
	\dGCANON expects as input a graph, whereas our row-column symmetry breaking circuits essentially solve the decision canonization for bipartite graphs.
	We proceed using the same encoding for graphs into bipartite graphs as discussed previously (see proof of Theorem~\ref{thm:core}).
	The order of a graph $\Gamma$ is fully determined by the order of its vertices. 
	However, observe that when only restricting the order of the left partition, the corresponding bipartite graph $\bip(\Gamma)$ may still differ in the order of the edges, i.e., in the order of the \emph{right} partition.  
	Indeed, the symmetry breaking circuit may choose to accept any of these orderings of the right partition.
	We define the set $\bip'(\Gamma)$ of bipartite graphs, where the left partition is ordered according to $V(\Gamma)$, and all potential reorderings of the right partition are contained.
	Note that $\bigcup_{\Delta \in \Gamma^{\Sym(V(\Gamma))}} \bip'(\Delta)$ covers all reorderings of the corresponding bipartite graphs.
	
	\textit{(Verifier.)} Given a certificate $\sigma$, our polynomial time verifier proceeds as follows:
	\begin{enumerate}
		\item For each column $c$ in the matrix of $x_{ij}$ variables, verify that the assignment has precisely two true variables in the column $c$. Formally, there exist $i, j \in [n]$ with $i \neq j$ such that $\sigma(x_{ic}) = 1$ and $\sigma(x_{jc}) = 1$, and for all $k \in [n]$ with $j \neq k \neq i$ it holds that $\sigma(x_{ic}) = 0$. 
		\item Check that $\Gamma$ corresponds to the bipartite graph as given by the assignment $\sigma$: for each edge $\{v_1, v_2\} \in E$, we verify that there exists a column $c$ such that $x_{v_1c} = 1$ and $x_{v_2c} = 1$. 
		\item Check if $\sigma$ satisfies $\delta_{n,m}(x_{11}, \cdots,x_{nm}, y_1,\dots,y_p)$. If it satisfies $\delta_{n,m}$, accept, otherwise reject. 
	\end{enumerate}
	By assumption, $\delta_{n,m}(x_{11},\dots,x_{nm}, y_1,\dots,y_p)$ can be computed in polynomial time. 
	Clearly, all the other steps can be computed in polynomial time as well. 

	\textit{(Correctness of Verifier.)} We need to argue that for each isomorphism class of graphs $\Gamma^{\Sym(V(\Gamma))}$, there is precisely one ordered graph $G$ accepted by the verifier.

	First, assume towards a contradiction that there is a $\Gamma$ such that for all $\varphi \in \Sym(V(\Gamma))$, and all certificates, $\Gamma^\varphi$ is rejected by the verifier.
	Consider the corresponding bipartite graph $\bip(\Gamma)$. 
	By assumption, we know that for all orderings of the left partition, all orderings of the right partition are rejected by the verifier.
	Hence, all orderings of $\bip(\Gamma)$ under $\Sym([n]) \times \Sym([m])$ are rejected by the symmetry breaking circuit.
	Hence, $\delta_{n,m}$ can not be a correct symmetry breaking circuit, which is a contradiction.
	
	Next, assume towards a contradiction that there are two distinct isomorphic graphs $\Gamma \cong \Delta$ which are both accepted by the verifier.
	Since $\delta_{n,m}$ is a correct symmetry breaking circuit, this may only occur if there are corresponding bipartite graphs $\Gamma^\ast \in \bip'(\Gamma), \Delta^\ast \in \bip'(\Delta)$ such that $\Gamma^\ast = \Delta^\ast$.
	However, this would immediately imply $\Gamma = \Delta$.
\end{proof}

\section{Johnson Actions} \label{sec:johnson}
Next, we consider the so-called Johnson groups.
Johnson groups are groups which naturally occur in problems encoding graph generation tasks \cite{DBLP:conf/cp/KirchwegerS21}.
We begin this section by describing a correspondence between Johnson groups and symmetric relational structures.
Then, we provide a formal proof of \cref{thm:johnson}.
Lastly, we show how to derive SBPs for a group $G$, given SBPs for a small index subgroup $H \leq G$ (\cref{lemma:index}).

\subsection{Johnson Groups and Relational Structures} \label{sec:johnsonstructure}

\xparagraph{Johnson Families.} 
Let $k$ be a positive integer. For $t \in [k-1]$, let $X^t_k$ be the set of variables indexed by $t$-element subsets of $[k]$. In particular, we have $|X^t_k| = \binom{k}{t}$.
For fixed~$t \geq 1$, we call the group family $S_k^{(t)} \leq \Sym(X_k^t)$ the \emph{Johnson family of arity~$t$}.

Johnson groups form a subclass of the so-called groups of Cameron type. These groups as well as their natural action can be recognized in polynomial time (see \cite{permNC}).

\xparagraph{Relational Structures and Johnson Groups.} To a symmetric $t$-ary relational structure~$R$, we associate an assignment $f_{R} \colon X^t_k \to \{0,1\}$ with $f(x_S) = 1$ for a $t$-subset $S$ of $[k]$ if $S$ is a hyperedge in $R$, and $f(x_S) = 0$ otherwise. Conversely, given $f \colon X^t_k \to \{0,1\}$, we define a symmetric $t$-ary relational structure $R_f$ on the universe $[k]$ whose relation is the set of all tuples $(a_1,\dots,a_t)$ with $f(\{a_1,\dots,a_t\}) = 1$.

This defines a one-to-one correspondence between assignments of $X^t_k$ and symmetric $t$-ary relational structures. The following result formalizes the correspondence (see also \cite{DBLP:conf/stoc/Luks99}). 
\begin{restatable}{lemma}{relationalstructures}\label{lemma:relationalstructures}
Let $R$ and $R'$ be two symmetric $t$-ary relational structures on the universe $[k]$.
Then $R$ and $R'$ are isomorphic if and only if the assignments $f_{R}$ and $f_{R'}$ of the set $X^t_k$ lie in the same orbit under the action of the Johnson group $S^{(t)}_k \leq \mathrm{Sym}(X^t_k)$. 
\end{restatable}

\begin{proof}
	Suppose that $R$ and $R'$ are isomorphic via a bijection $\pi \colon [k] \to [k]$. Then, the induced action $\hat{\pi}$ on $t$-subsets of $[k]$ defines an element of $S^{(t)}_k$ with $f_{R'} = f_{R}^{\hat{\pi}}$.
	Conversely, suppose that $f_{R'} = f_R^{\hat{\pi}}$ for some $\hat{\pi}$ corresponding to the induced action of $\pi \colon [k] \to [k]$. It is easy to check that~$\pi$ is an isomorphism between $R$ and~$R'$. 
\end{proof}

\subsection{Johnson Families of Fixed Arity are Hard} \label{sec:johnsonfixed}
In this section, we show that polynomial time symmetry breaking for Johnson families of fixed arity $t \geq 2$ implies $\GI \in \coNP$. 

\begin{theorem}\label{theo:johnson2}
	Suppose there exists a polynomial time algorithm for generating complete symmetry breaking circuits with additional variables for the Johnson family of arity 2. Then, \dGCANON$\in \NP{}$ and hence $\GI{} \in \coNP{}$.
\end{theorem}
\begin{proof}
We again make use of Lemma~\ref{lem:dcanonnp}, proving that polynomial time symmetry breaking circuits with additional variables for Johnson groups give rise to a non-deterministic polynomial time algorithm for decision graph canonization. Using similar arguments to Theorem~\ref{lem:gridhard}, this follows from Lemma~\ref{lemma:relationalstructures} and the fact that for two relational structures $R, R'$ it holds that $R = R'$ if and only if $f_R = f_{R'}$.
\end{proof}

We generalize Theorem~\ref{theo:johnson2} to arbitrary arity. 
\begin{restatable}{theorem}{johnsonthree} 
	Let $t \geq 2$ be a fixed arity.
	Suppose there exists a polynomial time algorithm for generating complete symmetry breaking circuits with additional variables for the Johnson family of arity $t$. Then, \dGCANON$\in \NP{}$ and hence $\GI{} \in \coNP{}$. \label{theo:johnson3}
\end{restatable}
\begin{proof}
	Again, Lemma~\ref{lem:dcanonnp} ensures that it suffices to prove \dGCANON$\in \NP{}$.
	From Lemma~\ref{lemma:relationalstructures}, it follows that it suffices to solve \dGCANON{} in non-deterministic polynomial time using a non-deterministic polynomial time oracle for decision canonization for uniform, symmetric $t$-ary relational structures. 

	\textit{(Graph to $t$-ary Structure.)} Given a graph $\Gamma = (V, E)$, we define a $t$-uniform relational structure $R_\Gamma$ as follows.
	Let $I \subseteq V$ be the set of isolated vertices. 
	We have $V(R_\Gamma) = \{r_u \colon u \in V\} \cup \{v_1, \dots, v_{t-2}, a,b\}$. 
	Observe that we added $t$ vertices, namely $v_1, \dots, v_{t-2}, a,b$.
	These vertices will be called \emph{bogus vertices}. 
	We presume the order $r_{v_1} \prec \dots{} \prec r_{v_n} \prec v_1 \prec \dots{} \prec v_{t-2} \prec a \prec b$ for the symbols used in the construction.
	The hyperedges in $R_\Gamma$ are given by 
	\begin{alignat*}{1}
		\Bigl\{\{r_u,r_w,v_1, \dots, v_{t-2} \} \colon \{u,w\} \in E\Bigr\} &\cup \Bigl\{\{r_u,v_1, \dots, v_{t-2}, a\} \colon u \in V \setminus I\Bigr\} \\
		&\cup \Bigl\{\{v_1, \dots, v_{t-2}, a,b\}\Bigr\}.
	\end{alignat*}
	
	Observe that \[
	\begin{cases}
		\deg_{R_\Gamma}(r_u) = 0, & u \in I \\
		\deg_{R_\Gamma}(r_u) = \deg_\Gamma(u) + 1,  & u \in V \setminus I \\
		\deg_{R_\Gamma}(v_i) = |E| + |V \setminus I| +1, & i \in [t-2]\\
		\deg_{R_\Gamma}(a) = |V\setminus I|+1, & \\
		\deg_{R_\Gamma}(b) = 1. & \\
	\end{cases}
	\] 
	In particular, for $u \in V \setminus I$, we have $1< \deg_{R_\Gamma}(r_u) \leq |V \setminus I|$. 
	
	Now let $\Gamma$ and $\Delta$ be graphs on $n$ vertices. 
	Without loss of generality, we may assume that $\Gamma$ and $\Delta$ contain edges. 
	We claim that $\Gamma$ and $\Delta$ are isomorphic precisely if $R_\Gamma$ and $R_\Delta$ are isomorphic. Assume that there exists an isomorphism $\varphi \colon R_\Gamma \to R_\Delta$. 
	Denote the vertices in $R_\Gamma$ and $R_\Delta$ with an exponent $\Gamma$ and $\Delta$, respectively. 
	By the above degree conditions, we have $\varphi(b^\Gamma) = b^\Delta$ (here, the notation $b^\Gamma$ refers to node $b$ of graph $\Gamma$). 
	As $b^\Gamma$ is adjacent to $v_1^\Gamma, \dots, v_{t-2}^\Gamma, a^\Gamma$ (similarly in $\Delta$), the degree conditions then imply $\varphi(a^\Gamma) = a^\Delta$. Now the vertices $v_1^\Gamma, \dots, v_{t-2}^\Gamma$ are mapped bijectively to $v_1^\Delta, \dots, v_{t-2}^\Delta$. 
	In particular, $\varphi$ induces a bijection between $\{r_u^\Gamma \colon u \in V(\Gamma)\}$ and $\{r_u^\Delta \colon u \in V(\Delta)\}$. 
	It is then easy to see that $\varphi$ induces an isomorphism between $\Gamma$ and~$\Delta$.

	On the other hand, if $\Gamma$ and $\Delta$ are isomorphic, $R_\Gamma \cong R_\Delta$ follows from the fact that the above construction is isomorphism-invariant: all additional bogus vertices universally appear with all edges, as well as with all non-isolated vertices.

	Furthermore, it is easy to see that if $\Gamma \neq \Delta$, then $R_\Gamma \neq R_\Delta$ follows.

	\textit{(Certificate.)} Our certificate will consist of a permutation $\varphi \in \Sym(V(R_\Gamma))$, as well as a certificate for decision canonization of $t$-ary structures $\sigma$.
	
	\textit{(Verifier.)} Our verifier proceeds as follows:
	\begin{enumerate}
		\item Using the decision canonization oracle for $t$-ary structures, continue if $\sigma$ is a valid certificate for $R_\Gamma^\varphi$, and reject otherwise.
		\item If for all pairs of vertices $v, v' \in V(\Gamma)$ with $v \prec v'$ it holds that $\varphi(r_{v}) \prec \varphi(r_{v'})$, accept, otherwise reject.
	\end{enumerate}

	\textit{(Correctness of Verifier.)} 
	From the arguments above, we know that for all graphs $\Delta$ in the isomorphism class of $\Gamma$ it holds that $R_\Gamma \cong R_\Delta$.
	The oracle in Step~1 will accept precisely one canonical $t$-ary structure $R_\Gamma^\varphi$ in the isomorphism class of $R_\Gamma$.   
	In turn, the verifier accepts a graph $\Gamma$, if and only if the order of the vertices is preserved in the canon $R_\Gamma^\varphi$ (see Step~2).
	We remark that there may also be different $\varphi'$ which map $R_\Gamma$ to the canon, which may not preserve the order of $V(\Gamma)$.
	Clearly, there is at least one graph $\Gamma$ in each isomorphism class that is accepted by the verifier.
	
	Assume there is another graph $\Delta \neq \Gamma$ with $\Delta \cong \Gamma$ which is also accepted by the verifier.
	Since we know that $R_\Delta \cong R_\Gamma$ holds, this means there is a $\varphi' \in \Sym(V(R_\Delta))$ such that $R_\Delta^{\varphi'} = R_\Gamma^{\varphi}$ holds.
	In particular, $\varphi'$ preserves the order of vertices in $\Delta$.
	Recall that bogus vertices can only ever be mapped to bogus vertices.
	Therefore, $R_\Delta^{\varphi'} = R_\Gamma^{\varphi}$ immediately implies that the vertices of $\Gamma$ and $\Delta$ can be mapped, in order, onto each other, while preserving the edge relation of the original graphs.
	In other words, $\Delta =  \Gamma$ holds, which is a contradiction to the assumption that the verifier accepts $\Delta$.
\end{proof}

\begin{remark}
In contrast, observe that the Johnson family for $ t= 1$ consists of the symmetric groups $\Sym(n)$ in their natural action on $n$ points. For these groups, complete symmetry breaking can be achieved with a CNF predicate of linear size (see Section~\ref{sec:upperbounds}). 
\end{remark}

\subsection{Subgroups of Small Index and Large Primitive Groups}\label{sec:babai}

In this section, we consider symmetry breaking for a permutation group $G \leq \Sym(n)$ and a subgroup~$H$ of $G$. Mostly, we are interested in the case that $H$ has polynomial index in $G$. We first show that a symmetry breaking constraint for~$H$ gives rise to symmetry breaking constraint for~$G$:
\begin{lemma}\label{lemma:index}
	There exists a polynomial $p$ such that the following holds: if there is a complete symmetry breaking circuit for a group $H \leq \Sym(n)$ which can be computed in time $t$, then complete symmetry breaking circuit with additional variables for $G \leq \Sym(n)$ with $G \geq H$ can be computed in time $t \cdot p(n|G:H|)$.
\end{lemma}

\begin{proof}
	Let $\psi$ be a symmetry breaking circuit for $H$. We now devise a symmetry breaking circuit for $G$. For simplicity, we fix a system of representatives $R$ of the right cosets of $H$ in~$G$, which can be computed in time polynomial in $|G:H|$ (see \cite{holt2005handbook}).
	
	\textit{(Certificate.)} 
	The certificate $\sigma = \{(\theta_r, h_r) \colon r \in R\}$ consists of assignments $\theta_r \colon \Var(\psi) \to \{0,1\}$ and an element $h_r \in H$ for every $r \in R$. 
	
	\textit{(Verifier.)} Given an assignment~$\theta \colon \Var(\psi) \to \{0,1\}$ and a certificate $\sigma = \{(\theta_r, h_r) \colon r \in R\}$, we proceed as follows: 
	\begin{enumerate}
	\item For all $r \in R$, verify that $\theta_r^{h_r} =  \theta^r$ holds. 
	\item For all $r \in R$, verify that $\psi[\theta_r]$ is satisfying.  Verify that $\psi[\theta]$ is satisfying. 
	\item For all $r \in R$, check whether $\theta \preceq_{\text{lex}} \theta_r$ holds. If this is the case, accept $\theta$, otherwise reject. 
	\end{enumerate}
	Clearly, the runtime of this procedure is polynomial in $t$ and $|G:H|$. 
	
	\textit{(Correctness of Verifier.)} Let $\Delta$ be a $G$-orbit of assignments. Note that $\Delta$  is a disjoint union of $H$-orbits $\Delta_1, \dots, \Delta_k$. In each $\Delta_i$, there exists a unique assignment $\alpha_i$ such that $\psi[\alpha_i]$ is satisfying. Let $\theta$ be the lexicographically minimal element in $\{\alpha_1, \dots, \alpha_k\}$. Note that we have 
	$\Orbi{G}{\theta} = \bigcup_{r \in R} \Orbi{H}{(r \theta)}$
as every element of $G$ can be decomposed as $hr$ for $h \in H$ and $r \in R$. For $r \in R$, there exists $i_r \in [k]$ with $\Orbi{H}{(r\theta)} = \Delta_{i_r}$. Hence, there exists $h_r \in H$ with $\alpha_{i_r}^{h_r} = \theta^r$. By construction, $\theta$ together with the certificate $\sigma = \{(\alpha_{i_r}, h_r) \colon r \in R\}$ is accepted by the verifier.
	
	Now suppose that $\theta, \theta' \in \Delta$ are accepted by the verifier, and let $\sigma_\theta = \{(\theta_r, h_r) \colon r \in R\}$ and $\sigma_{\theta'} = \{(\theta'_r, h_r') \colon r \in R\}$ denote corresponding certificates. 
	Due to the decomposition of $\theta^G$ and since $\psi[\theta_r]$ and $\psi[\theta'_r]$ are satisfying for all $r \in R$, we have $\{\alpha_1, \dots, \alpha_k\} = \{\theta_r \colon r \in R\} = \{\theta_r' \colon r \in R\}$.  
	Since the verifier accepts both $\theta$ and $\theta'$, they coincide with the lexicographically minimal element in $\{\alpha_1, \dots, \alpha_k\}$, so $\theta = \theta'$ follows.
\end{proof}
It should be noted that while the above lemma gives a valid upper bound, the resulting SBP is not practical: The SBP simply uses the additional variables to determine the representative for all cosets, and then determines a minimal one among them. 
This requires trying out all the symmetric choices, defeating the purpose of the SBP.
However, the result can also be read as a hardness result. 
For example, for the matrix models studied in Section~\ref{sec:gridhardness}, we can restrict the group on each axis of the model as follows, while still being able to retrieve our hardness result (see Theorem~\ref{lem:gridhard}):
\begin{corollary}
	Consider a family of permutation groups $G_{m,n} = X_m \times Y_n$ with $X_m \leq \Sym(m)$ and $Y_n \leq \Sym(n)$, acting component-wise on $[m]\times [n]$. Assume that $|\Sym(m) : X_m| < \poly(m)$ and $|\Sym(n): Y_n| < \poly(n)$ holds. Then, efficient complete symmetry breaking with additional variables for $G_{m,n}$ implies $\GI{} \in \coNP{}$.
\end{corollary}

Our main interest in studying subgroups of small index is sparked by a result on the structure of so-called large primitive groups, which forms an important building block of the quasi-polynomial isomorphism test for general graphs \cite{DBLP:conf/stoc/Babai16}. Roughly speaking, every primitive group $G \leq \Sym(n)$ with $|G| \geq n^{1 + \log_2 n}$ contains a normal subgroup $N$ with $|G:N| \leq n$ exhibiting a natural Johnson action. If the converse of Lemma~\ref{lemma:index} holds, we can thus employ our results on Johnson groups to study the complexity of symmetry breaking for large primitive groups.

\section{Upper Bounds} \label{sec:upperbounds}
Complementing the results from the previous sections, we show that certain families of groups can be efficiently handled. 
We begin by recalling three simple cases. 

\xparagraph{Groups of Polynomial Order.} 
The first case pertains to groups where the order is polynomial in the size of the original formula. 
For these groups, we can explicitly write a constraint that breaks each element of the group \cite{DBLP:conf/kr/CrawfordGLR96}.
The resulting constraint is complete and of polynomial size in the formula.

\xparagraph{Symmetric Groups.} 
Symmetric groups in their natural action can be handled by imposing an ordering on the assignments. For $\Sym(n)$, this can be achieved by the predicate $\psi_n = x_1 \leq x_2 \leq \dots \leq x_n$. 

A slight extension of symmetric groups are known and used in practice, namely row-interchangeability subgroups \cite{DBLP:conf/sat/Devriendt0BD16, DBLP:journals/mpc/PfetschR19}. 
We say that a permutation group $G \leq \Sym(\Omega)$ exhibits \emph{row-interchangeability} if $\Omega$ can be arranged in an $n \times m$-matrix $X = (x_{ij})$ such that $G$ consists precisely of the permutations of the \emph{rows} of $X$.
This symmetry can be broken by lexicographically ordering the rows in any assignment $\theta \colon X \to \{0,1\}$ (viewed as an $n \times m$-matrix). Formally, for $i \in [n-1]$, let $\lambda_i^k \coloneqq (\bigwedge_{r \in [k-1]} (x_{ir} = x_{(i+1)r})) \rightarrow (x_{ik} \leq x_{(i+1)k})$. Then 
$\lambda_{n,m} \coloneqq \bigwedge_{i = 1}^{n-1} \bigwedge_{k = 1}^m \lambda_i^k$
is a symmetry breaking predicate for~$G$.

\xparagraph{Disjoint direct decomposition.}
A direct product $G = G_1 \times \dots  \times G_r$ of permutation groups is called a \emph{disjoint direct decomposition} if the subgroups $G_1, \ldots, G_r$ have pairwise disjoint supports.
Disjoint direct products naturally arise and have been successfully used in practice \cite{DBLP:conf/sat/Devriendt0BD16}. 
The finest disjoint direct decomposition can be computed in polynomial time for general permutation groups \cite{DBLP:journals/jsc/ChangJ22}, and in quasi-linear time for automorphism groups of graphs \cite{DBLP:conf/sat/AndersSS23}. For the sake of completeness, we argue that disjoint direct decompositions can be exploited without giving up on complete symmetry breaking.

\begin{restatable}{lemma}{disjointdecomposition}\label{lemma:disjointdecomposition}
	Let $G \leq \Sym(\Omega)$ be a permutation group with a disjoint direct product decomposition $G = G_1 \times \dots \times G_r$. 
	For $i \in [r]$, let $\Omega_i$ denote the support of $G_i$ and assume that a complete symmetry breaking predicate $\mgG_i$ for $G_i$, viewed as a permutation group on~$\Omega_i$, is given. In particular, we require $\Var(\mgG_i) \subseteq \Omega_i$. Then $\mgG \coloneqq \mgG_1 \land \dots \land \mgG_r$ is a complete symmetry breaking predicate for~$G$.
\end{restatable}
\begin{proof}
	Let $F \subseteq \Omega$ be the set of points fixed by $G$. Then $\Omega = \Omega_1 \dcup \dots \dcup \Omega_r \dcup F$.
	An assignment $\theta \colon \Omega \to \{0,1\}$ can be viewed as a tuple $(\theta_1, \dots, \theta_r, \theta_F)$ of assignments defined on $\Omega_1, \ldots, \Omega_r, F$, respectively, and we have $\Orbi{G}{\theta} = \Orbi{G_1}{\theta_1} \times \dots \times \Orbi{G_r}{\theta_r} \times \{\theta_F\}$. Hence $\theta$ satisfies~$\mgG$ if and only if $\theta_i$ satisfies $\mgG_i$ for every $i \in [r]$. Thus, $\mgG$ is a complete symmetry breaking predicate for $G$.
	\end{proof}
The size of the constraint $\mgG$ is linear in the size of the constraints $\mgG_1, \dots, \mgG_r$. 

\xparagraph{Wreath Products.}
We now turn our attention to so-called \emph{wreath products}.
They naturally occur as the automorphism groups of tree-like structures and can be detected as the induced action on a block system \cite{seress_2003}. 
Tree-like appendages are already detected and exploited by practical symmetry detection algorithms \cite{DBLP:conf/wea/AndersSS23}, and thus these wreath products seem readily available.

Indeed, certain cases of wreath products can be efficiently handled in symmetry breaking. 
Specifically, we show that automorphism groups of trees can be taken care of (see Theorem~\ref{thm:tree-intro}). 

Intuitively, a wreath symmetry occurs if the domain can be partitioned into equally-sized parts with identical symmetries that can be permuted among each other. 
The corresponding symmetry group is made of a group describing the possible permutations of the points within a part, and a group describing the permutation of the parts. 
Formally, let $G \leq \Sym(n)$ and $H \leq \Sym(m)$, and consider the imprimitive action of $G \wr H$ on $X \coloneqq \{x_{ij} \colon i \in [n],\, j \in [m]\}$. Explicitly, it is given by 
$x_{ij}^{((g_1, \dots, g_m), h)} = x_{g_{h(j)}(i) \,h(j)}$. 
For $\theta \colon X \to \{0,1\}$ and $j \in [m]$, let $\theta_{j} \coloneqq \theta|_{\{x_{1j}, \ldots, x_{nj}\}}$
and write $\theta = (\theta_1, \dots, \theta_m)$. 

\xparagraph{Wreath Products with CNF.} Let us first focus on CNF predicates. Recall the predicate $\lambda_{m,n}$ from the beginning of this section. The following result shows that a symmetry breaking predicate for a permutation group $G$ can be ``lifted'' to a predicate for $G \wr \Sym(m)$: 

\begin{lemma}\label{prop:wreathproduct} 
	Assume that $\gamma$ is a complete symmetry breaking predicate for $G \leq \Sym(n)$ and set $\gamma_j \coloneqq \gamma(x_{1j}, \dots, x_{nj})$ for all $j \in [m]$. 
	Then $\nu \coloneqq \bigwedge_{j \in [m]} \gamma_j \land \lambda_{m,n}$ is a complete symmetry breaking predicate for $W \coloneqq G \wr \Sym(m)$. 
\end{lemma} 
\begin{proof}
	Let $\theta \colon X \to \{0,1\}$ be an assignment. For every $j \in [m]$, there exists $g_j \in G$ such that $\theta_{j}^{g_j}$ satisfies $\gamma$. Write $\theta' \coloneqq \theta^{((g_1, \dots, g_m),1)}$. There exists $h \in \Sym(m)$ with $\theta'_{h^{-1}(1)} \preceq_{\text{lex}} \dots \preceq_{\text{lex}} \theta'_{h^{-1}(m)}$.  Hence, the assignment $\theta'^{(1,h)} \in \Orbi{W}{\theta}$ satisfies~$\nu$. 
	
	On the other hand, consider assignments $\theta, \theta' \colon X \to \{0,1\}$ satisfying~$\nu$, and assume $\theta' = \theta^{((g_1, \dots, g_m), h)}$ for $\bigl((g_1, \dots, g_m), h\bigr) \in W$. For all $j \in [m]$, this implies $\theta'_{j} =  \theta_{h^{-1}(j)}^{g_{h^{-1}(j)}}\in \Orbi{G}{\theta_{h^{-1}(j)}}$. As $\theta'_j$ and $\theta_{h^{-1}(j)}$ satisfy $\gamma$, they coincide, so we may choose $g_1 = \dots = g_m = 1$. Since $\theta$ and $\theta'$ satisfy $\lambda_{m,n}$, we have $\theta_1 \preceq_{\text{lex}} \dots \preceq_{\text{lex}} \theta_m$ and $\theta_1' \preceq_{\text{lex}} \dots \preceq_{\text{lex}} \theta_m'$. This yields $\theta_j' = \theta_j$ for all $j \in [m]$, so $\theta = \theta'$ follows.
\end{proof}

\begin{remark}\label{rem:sizepredicate}
	The size of the predicate $\nu$ given in Lemma~\ref{prop:wreathproduct} is in $\mathcal{O}(s(\mgG)  m + nm)$, where $s(\mgG)$ denotes the size of $\mgG$. Note that if $s(\mgG) \in \mathcal{O}(n)$ holds, then the size of $\nu$ is linear in the domain size $nm$ of the wreath product. 
\end{remark}

\begin{corollary}\label{cor:wrsymmetricgroups}
	The predicate $\nu = \bigwedge_{j \in[m]} (x_{1j} \leq \dots \leq x_{nj}) \land \lambda_{m,n}$
	is a complete symmetry breaking predicate for $\Sym(n) \wr \Sym(m)$. 
\end{corollary}

Combining the results for direct disjoint decompositions and wreath products, it readily follows that automorphism groups of trees can be handled efficiently (see Theorem~\ref{thm:tree-intro}). 

\begin{proof}[Proof of Theorem~\ref{thm:tree-intro}]
The group $G$ can be constructed by iterated disjoint direct decompositions and wreath products in which the top group is a full symmetric group \cite{10.1007/BF02546665}. Combining Lemma~\ref{lemma:disjointdecomposition} and Proposition~\ref{prop:wreathproduct} thus yields a symmetry breaking predicate for~$G$. Inductively, it follows from Remark~\ref{rem:sizepredicate} that the size of this predicate is linear.
\end{proof}

\xparagraph{Wreath Products with Circuits.}
Using circuits and a few further assumptions, we treat general wreath products $W \coloneqq G \wr H$.
\wreath*
\begin{proof}
	Since we may turn a polynomial time algorithm into a polynomial-sized circuit, it suffices to give a polynomial-time algorithm for symmetry breaking for~$W$. 
	
	Let $\psi_G$ denote the symmetry breaking circuit for $G$, and for any partition stabilizer~$S$ in $H$, write $\psi_S$ for the corresponding symmetry breaking circuit. For an assignment $\theta \colon X \to \{0,1\}$, write $ \theta = (\theta_1, \dots, \theta_m)$ as before. For $i \in [n]$, let $c_i(\theta)$ be the string of length~$m$ consisting of the $i$-th entries of $\theta_1, \dots, \theta_m$. We define partitions $P_1, \dots, P_n$ of $[m]$ as follows: let $P_1$ denote the partition into the index sets of zero and one entries in $c_1(\theta)$. For $i \geq 2$, $P_i$ is the refinement of $P_{i-1}$ according to the zero-one-partition of $c_i(\theta)$. For $i \in [n]$, let $S_i$ denote the partition stabilizer of $P_i$ in $H$, and set $S_0 \coloneqq H$.
	
	\textit{(Description of Algorithm.)} Given an assignment $\theta = (\theta_1, \dots, \theta_m)$, we define our algorithm as follows: 
	\begin{enumerate}
	\item If $\psi_G[\theta_i]$ is non-satisfying for some $i \in [m]$, return false. 
	\item For $i \in [n]$, compute the vectors $c_i(\theta)$ as well as the partitions $P_i$ and their stabilizers~$S_i$. 
	\item For $i \in [n]$, check if $\psi_{S_{i-1}}[c_i(\theta)]$ is satisfying. If this fails for some $i \in [n]$, return false. Otherwise, return true. 
	\end{enumerate}
	
	\textit{(Correctness of Algorithm.)} 
	By assumption, partition stabilizers in $H$ as well as all the necessary symmetry breaking circuits can be computed in polynomial time. 
	The remaining steps of the algorithm can clearly be computed in polynomial time.
	
	Replacing the input assignment $\theta = (\theta_1, \dots, \theta_m)$ by some element $\theta^{((g_1, \dots, g_m), 1)} \in \Orbi{W}{\theta}$, we may assume that $\psi_G[\theta_1], \dots, \psi_G[\theta_m]$ are satisfying. By assumption, there exists $h_1 \in H$ such that $\psi_H[c_1\bigl(\theta^{(1,h_1)}\bigr)]$ is satisfying. Moreover, there exists $h_2 \in S_1$ such that $\psi_{S_1}[c_2\bigl( \theta^{(1,h_2h_1)}\bigr)]$ is satisfying. Note that $c_1\bigl(\theta^{(1,h_2 h_1)}\bigr) = c_1\bigl( \theta^{(1,h_1)}\bigr)$ holds due to $h_2 \in S_1$. Continuing, we obtain an element $\theta' \coloneqq \theta^{(1,h_{n-1} \cdots h_1)}  \in \Orbi{W}{\theta}$ for which the algorithm returns true. 
	
	On the other hand, suppose that $\theta = (\theta_1, \dots, \theta_m)$ and $\theta' = (\theta_1', \dots, \theta_m')$ are assignments in the same $W$-orbit accepted by the algorithm.
	Then $\psi_G[\theta_i]$ and $\psi_G[\theta_i']$ are satisfying for all $i \in [m]$. Since $\theta$ and $\theta'$ lie in the same $W$-orbit, the strings $\theta_1, \dots, \theta_m$ and $\theta_1', \dots, \theta_m'$ coincide up to reordering. 
	The ordering of the substrings is lexicographic with respect to a successive application of~$H$. This yields $\theta = \theta'$. 
\end{proof}
In general, it is unknown whether partition stabilizers can be efficiently computed (see~\cite{holt2005handbook}). 
However, for $H = \Sym(m)$, the stabilizer of $P = (\Omega_1, \dots, \Omega_r)$ is simply given by $\Sym(\Omega_1) \times \dots \times \Sym(\Omega_r) \leq \Sym(n)$, and hence readily computable. 
This also holds if the order of $H$ is small. 
There, we obtain the following consequence of the preceding result:

\begin{corollary}
	Let $G \leq \Sym(n)$ and $H \leq \Sym(m)$ be permutation groups. Assume that a complete symmetry breaking circuit for $G$ can be computed in polynomial time and that $|H| \in \poly(n,m)$ holds. Then a complete symmetry breaking circuit
	for $G \wr H$ can be computed in polynomial time.
\end{corollary}

\section{Conclusion and Future Work}

We laid the foundation for a systematic study of the complexity of symmetry breaking.
A central tool in our investigation was the relation to the string canonization problem (see Section~\ref{sec:gridhardness}).
In particular, we showed that polynomial time symmetry breaking for the row-column symmetry group, even with circuits and additional variables, implies $\GI{} \in \coNP$ (see Theorem~\ref{thm:core}). 
The same applies to proper Johnson actions (see Theorem~\ref{thm:johnson}). 
On the other hand, we showed that symmetry breaking in polynomial time is possible for several families of groups, including certain classes of wreath products (see Section~\ref{sec:upperbounds}). 

Clearly, the dividing line between permutation groups that are ``hard'' instances for symmetry breaking, and those which admit efficient symmetry breaking, needs to be further explored. 
Among others, the following questions immediately arise: 
\begin{enumerate}
\item Given a permutation group $G$ and a subgroup $H$ of polynomial index, does~$H$ admit efficient symmetry breaking if $G$ does (i.e., does the converse of Lemma~\ref{lemma:index} hold)?
\item How difficult are permutation groups of intermediate size, in particular groups of quasi-polynomial order in the size of the domain?
\end{enumerate}

A positive answer to Question~1 would not only contribute to further decomposition results, but it is particularly relevant as large primitive permutation groups are known to contain normal subgroups of small index exhibiting a Johnson action.

\section*{Funding}
The research leading to these results has received funding from the European Research Council (ERC) under the European Union's Horizon 2020 research and innovation programme
(EngageS: grant agreement No. 820148). Sofia Brenner additionally received funding from the German Research Foundation DFG
(SFB-TRR 195 “Symbolic Tools in Mathematics and their Application”).

\small
\bibliography{arxiv}
\bibliographystyle{plainurl}

\end{document}